\documentclass{article}
\usepackage{spconf,amsmath,graphicx}

\usepackage{amsmath}
\usepackage{amsfonts}
\usepackage{amsthm}

\newtheorem{thmx}{Theorem}

\renewenvironment{proof}{{\bf Proof:}}{\qed}

\usepackage{mathtools}

\usepackage{hyperref}

\usepackage{algorithm}
\usepackage[noend]{algorithmic}

\title{Towards a Spectrum of Graph Convolutional Networks}
\name{Mathias Niepert and Alberto Garc{\'{\i}}a{-}Dur{\'{a}}n}
\address{NEC Labs Europe}
\begin{document}
%\ninept
%
\maketitle
\begin{abstract}
We present our ongoing work on understanding the limitations of graph convolutional networks (GCNs) as well as our work on generalizations of graph convolutions for representing more complex node attribute dependencies. Based on an analysis of GCNs with the help of the corresponding computation graphs, we propose a generalization of existing GCNs where the aggregation operations are (a) determined by structural properties of the local neighborhood graphs and (b) not restricted to weighted averages. We show that the proposed approach is strictly more expressive while requiring only a modest increase in the number of parameters and computations. We also show that the proposed generalization is identical to standard convolutional layers when applied to regular grid graphs. 
\end{abstract}
\begin{keywords}
graph convolutional networks, graph partition, convolutions
\end{keywords}
\section{Introduction}
\label{sec:intro}

With this paper we present our ongoing work on exploring a spectrum of graph convolutional networks that addresses some of the shortcomings of GCNs as introduced in the context of graph convolutional learning~\cite{kipf2016semi,monti2016geometric}. We provide some background on graph convolutional networks, and explore GCNs with the help of computation graphs. We then introduce a novel way to generalize GCNs. This is done by clustering the neighborhood of each node into exactly $c$ components. Each of these components determines the nodes whose attribute values are aggregated into a $c$-dimensional vector. This introduces more complex feature maps and we can show that for regular graphs the method is as expressive as standard 2D convolutional networks. At the same time, we can show that the GCN as introduced by Kipf \& Welling and the MoNet framework are special instances of the proposed approach where the attribute values of all nodes in the 1-hop neighborhood are average pooled based on global node degree information. 
The limitations of GCNs as introduced by Kipf \& Welling on regular graphs, that is, graphs with a regular neighborhood connectivity such as grid graphs, have been explored before\footnote{http://www.inference.vc/how-powerful-are-graph-convolutions-review-of-kipf-welling-2016-2/}. Our motivation is an exploration of generalizations of GCNs. 

Since this is a paper describing ongoing work, we do not provide experimental results. We state, however, some results on the expressivity of the method and its relationship to existing approaches. We also believe that the discussion of GCNs in terms of computation graphs is helpful for understanding GCNs and their relation to other graph-based convolutional network approaches which perform an ordering operation in node neighborhoods~\cite{niepert2016learning}. 

\section{Related Work}
\label{related_work}

There has been a large number of methods for learning graph representations. These include not only approaches to the node classification problem~\cite{perozzi2014deepwalk,tang2015line,yang2016revisiting,kipf2016semi,ep2017} but also approaches for the problem of classifying entire graphs~\cite{yanardag2015deep,duvenaud2015convolutional,li2015gated,niepert2016learning,defferrard2016convolutional,bronstein2017geometric}. With our work we focus on the problem of node representation learning and, specifically, graph convolutional networks applied to this problem. 

There is a large number of graph-based learning frameworks such as graph neural networks~(\textsc{GNN}) \cite{scarselli2009graph}, gated graph sequence neural networks (\textsc{GG-SNN})~\cite{li2015gated}, diffusion-convolutional neural networks  (\textsc{DCNN})~\cite{atwood2016diffusion}, and graph convolutional networks (\textsc{GCN}) \cite{kipf2016semi}, to name but a few. Most of these methods can be described as instances of the Message Passing Neural Network (\textsc{Mpnn}) framework~\cite{gilmer2017neural} where messages are passed between nodes to update their representations. Moreover, several more recent approaches such as GCNs~\cite{kipf2016semi}, DCNNs~\cite{atwood2016diffusion}, and graph attention networks~\cite{velivckovic2017graph} are special cases of MoNet~\cite{monti2016geometric}.
The training of these instances is guided by a supervised loss and the methods were shown to have state-of-the-art performance on a diverse set of learning problems.

There are several approaches defining convolutions by applying operations on groups
of neighbors. One of the challenges all of these methods have in common is to apply convolutions to neighborhoods of varying sizes so as to preserve the weight sharing property of standard CNNs for regular graphs. A seminal strategy involved the learning of a weight matrix for each possible node degree~\cite{duvenaud2015convolutional}. 
An additional recent idea applies self-attention to the graph learning domain~\cite{velivckovic2017graph}.

There are also several novel unsupervised learning approaches suitable to address node classification problems. \textsc{DeepWalk} \cite{perozzi2014deepwalk} collects random walks on the graphs and applies a \textsc{Skipgram} \cite{mikolov2013distributed} model to learn node representations. Despite its conceptual simplicity, it was shown to outperform spectral clustering approaches with the  advantage of scaling to large graphs.  %\textsc{LINE}~\cite{tang2015line} optimizes similarities between pairs of node embeddings so as to preserve their first and second-order proximity.
More recent work \cite{hamilton2017inductive,ep2017} introduce unsupervised learning frameworks that can incorporate node attributes.

\section{Background}

%\subsubsection*{Convolutional Networks}
A predecessor to CNNs was the Neocognitron~\cite{Kunihiko:1980}. %A typical CNN is composed of alternating convolution and pooling layers. 
A typical (deep) CNN is composed of convolutional and dense layers. The purpose of the convolutional layers is the extraction of common patterns found within local regions of the input images. CNNs convolve learned filters over the input image, computing the inner product at every location in the image and outputting the result as tensors. Figure~\ref{fig-conv} illustrates a  convolutional layer. 

%\subsubsection*{Graphs and Labelings}
A graph $G$ is a pair $(V, E)$ with $V = \{v_1, ..., v_n\}$ the set of vertices and $E \subseteq V \times V$ the set of edges. Let $n$ be the number of vertices and $m$ the number of edges. %An undirected graph is a graph in which edges have no direction.
 Each graph can be represented by an adjacency matrix $\mathbf{A}$ of size $n \times n$, where $\mathbf{A}_{i,j} = 1$ if there is an edge from vertex $v_i$ to vertex 
$v_j$, and $\mathbf{A}_{i,j} = 0$ otherwise. Node and edge attributes are features that attain one value for each node and edge of a graph. We use the term attribute value instead of label to avoid confusion with the graph-theoretical concept of a labeling. 

To compute partitions of node neighborhoods, we utilizes graph labeling methods to impose an order on nodes. 
A graph labeling $\ell$ is a function $\ell: V \rightarrow S$  from the set of vertices $V$ to an ordered set $S$ such as the real numbers and integers. A graph labeling procedure computes a graph labeling for an input graph. A ranking (or coloring) is a function $\mathbf{r}:V \rightarrow \{1, ..., |V|\}$. Every labeling induces a ranking $\mathbf{r}$ with $\mathbf{r}(u) < \mathbf{r}(v)$ if and only if $\ell(u)>\ell(v)$.  Every graph labeling induces a partition $\{V_1, ..., V_n\}$ on $V$  with $u, v \in V_i$ if and only if $\ell(u)=\ell(v)$. 
Examples of graph labeling procedures are node degree and other measures of centrality commonly used in the analysis of networks. 
%Certain labeling procedures are also often used as a perprocessing step for graph canonicalization algorithms. 
A canonicalization of a graph $G$ is a graph $G'$ with a fixed vertex order which is isomorphic to $G$ and which represents its entire isomorphism class. In practice, the graph canonicalization tool \textsc{Nauty} has shown remarkable performance~\cite{McKay:2014}.

\begin{figure}[t]
\centering
\includegraphics[width=0.9\columnwidth]{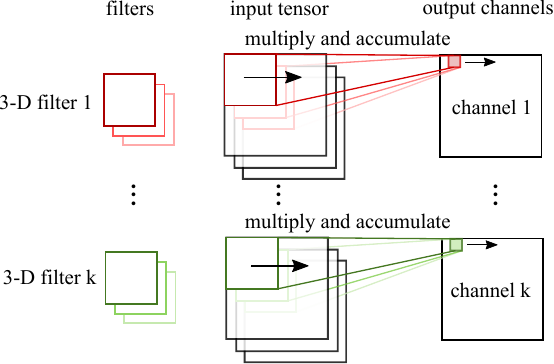}
\caption{\label{fig-conv} Each convolutional layer moves a set
          of 2-D filters (each being part of one 3-D filter) over the
          input channels from left to right and top to bottom and
          multiplies and accumulates the corresponding values. }
\end{figure}

\section{Graph Convolutional Networks}

Convolutional neural networks have been successfully applied to a large number of problems in areas such as computer vision and natural language processing. Figure~\ref{fig-conv} depicts the working of a typical convolutional network. Since a straight-forward application of convolutions is only possible in domains where the data can be described with a regular graph such as a grid graph modeling the pixels of an image, there has been recent work on extending ideas from convolutional network to irregular domains. 

In that line of work, graph convolutional networks (GCNs) as proposed by Kipf and Welling are a class of graph-based machine learning models for semi-supervised learning. GCNs are highly competitive and were shown to outperform Chebyshev~\cite{defferrard2016convolutional}, Planetoid~\cite{yang2016revisiting}, and EmbedNN~\cite{weston2012deep} and to be competitive with MoNet~\cite{monti2016geometric} on the commonly used benchmark data sets Cora, Citeseer, and Pubmed.

The GCNs introduced in previous work~\cite{kipf2016semi} can be characterized by the expression 
$$\mathbf{H}^{(l+1)} = \sigma( \hat{\mathbf{A}} \mathbf{H}^{(l)}\mathbf{W}^{(l)}),$$
where $\hat{\mathbf{A}}$ is some normalized variant of the adjacency matrix, $\mathbf{H}^{(l)}$ contains the vector representation (the feature map) of the graph vertices in the $l$-th layer, $\mathbf{W}^{(l)}$ is the weight matrix of the $l$th layer, and $\sigma$ a non-linearity.

While GCNs are highly useful for the type of benchmark data sets commonly used in the literature (e.g. the citation networks with text data Cora and Pubmed), they have some limitations. To better understand these limitations it helps to explore GCNs with the corresponding computation graph. It is straight-forward to show that the above formulation of GCNs is equivalent (up to a slightly different pooling operation) to the one depicted in Figure~\ref{fig-pool-gcn}. A pooling operation is an operation that takes a set of $n$-dimensional vectors and outputs an $n$-dimensional vector and can be more complex than average or max pooling. Each of the layers performs, for each node $v$ in the graph, a pooling operation that includes all 1-hop neighbors of $v$ as well as $v$ itself. Hence, there is one aggregate value computed for each input feature of that layer. The various 2D filters of the conv layers, therefore, have always dimension $a \times 1$ where $a$ is the number of input features (attributes). This makes the modeling of more complex dependencies between the features of nodes in the 1-hop neighborhood challenging if not impossible. Again, we believe that Figure~\ref{fig-pool-gcn} is helpful when trying to understand the limitations of Kipf \& Wellings GCNs. MoNet~\cite{monti2016geometric} generalizes this pooling operation by performing several weighted averages where the weights are computed from Gaussian distributions. Similar to GCNs, however, the weights are determined only by the global node degree information from neighboring nodes. 

\begin{figure}[t!]
	\centering
	\includegraphics[width=0.82\columnwidth]{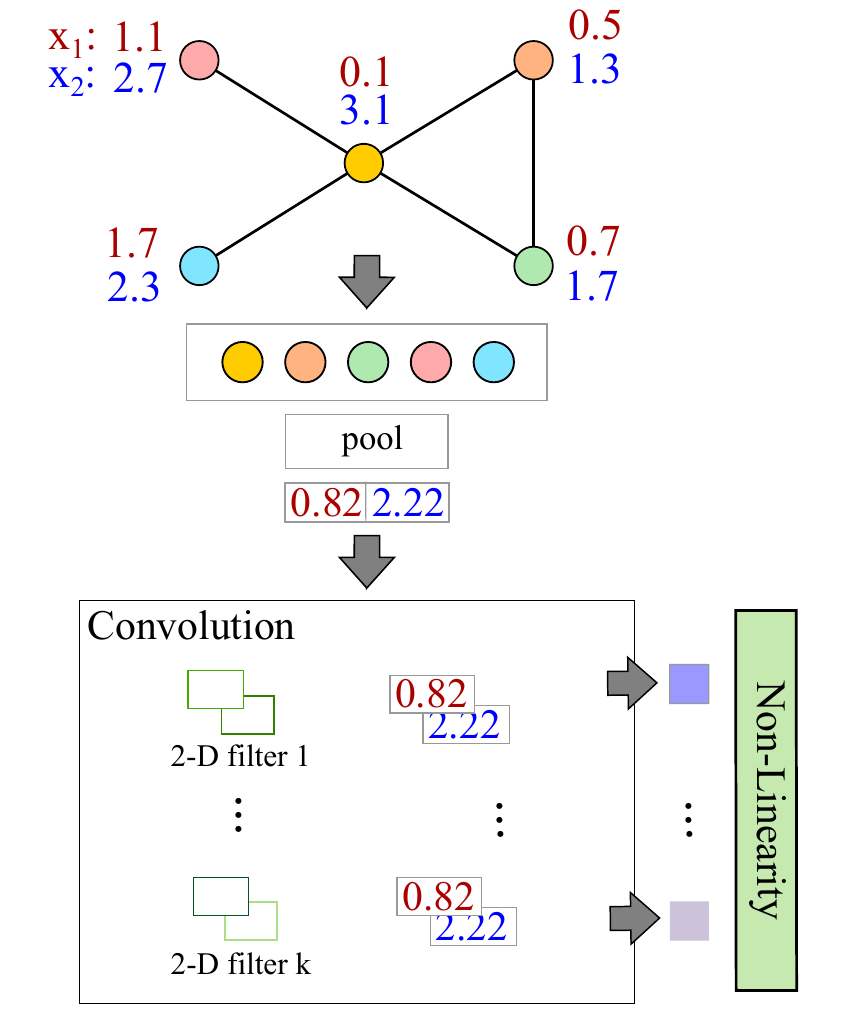}
	\caption{\label{fig-pool-gcn} Illustration of the GCN pooling and convolution operations in previous work~\cite{kipf2016semi}. For each node in the graph, the values of each attribute and all nodes in the 1-hop neighborhood are pooled. For the sake of simplicity, we use average pooling in the illustration. The pooling operation in Kipf \& Welling is fixed for each node and computed from the global node degrees. }
\end{figure}

\section{Generalizing GCNs}

We explore generalizations of the GCN of Kipf \& Welling and the generic MoNet framework. One such generalization involves the computation of a partition (or clustering) of $\mathbf{N}(v)$, the neighborhood of $v$, and using this partition to guide the aggregation operations. Instead of aggregating the feature values of all nodes before applying impoverished convolutions, we only aggregate the feature values of the nodes in each of the components of the partition. Hence, instead of only one aggregated value per feature, the input to the convolutions is now a vector of values on which filters can operate. In the following we describe the different changes and components of the proposed generalization. Figure~\ref{fig-pool-partition} illustrates the core ideas again using the corresponding computation graph.

\subsection{Neighborhood Clustering}

To apply more complex convolutional filters to the neighborhood of a given vertex $v$, we compute a partition or (soft) clustering of the \emph{local} neighborhood and obtain a more fine-grained set of aggregation operations. This shifts the problem to finding efficient and reasonable methods for generating such clusterings. Fortunately, there are efficient graph labeling algorithms that compute scores for vertices in a graph based on their structural roles. For instance, centrality measures score nodes according to their roles in the graph. The approach we propose here is motivated by the objective to compute clusterings such that the nodes in one cluster fulfill similar structural roles in the graph induced by $\mathbf{N}(v)$. 
\begin{algorithm}[t!]
   \small
   \caption{\label{alg:partition}Structural Partitioning}
\begin{algorithmic}[1]
   \STATE {\bfseries input:} vertex $v$, graph labeling method $\ell$, partition size $c$
   \STATE {\bfseries output:} $c$-partition of v's 1-hop neighbors $\mathbf{N}(v)$ 
   \STATE create subgraph $G' = (E', V')$ induced by $v$ and its 1-hop neighborhood
   \STATE compute values $\mathbf{r}$ of $V'$ by applying $\ell$ to $G'$
   \STATE set $C_1 := \{v\}$
   \STATE apply clustering approach (k-means, ...) to $\mathbf{r}$ with $c-1$ centers
   \STATE assign vertices to the clusters $\{C_2, ..., C_c\}$
   \STATE {\bfseries return}  $\{C_1, ..., C_c\}$
\end{algorithmic}
\end{algorithm}
Hence, for each node $v$ in the graph and as a preprocessing step, we apply a graph labeling method to the graph \emph{induced by the neighborhood} of $v$. The resulting labeling (ranking) is used to construct the partition of the neighbors. Note that node labeling approaches such as Weisfeiler-Lehman~\cite{weisfeiler:1968} and centrality measures provide a strictly more fine-grained partition/clustering of the nodes compared to the node degree used in GCNs and MoNet. Moreover, the weighted averaging performed by GCNs and MoNet  is based on the global node degree and not structural measures applied to the local neighborhood graphs. 
Algorithm~\ref{alg:partition} lists the procedure for computing a partition for each of the nodes in the graph. Again, the idea is that the partition lumps together nodes with similar structural roles. The aggregation operation is then only performed on sets of these more similar nodes. 
As alternatives to labeling methods such as centrality measures, one could also use the attribute values or some existing domain knowledge to compute the partitions. 

The partitioning can be integrated into an end-to-end differentiable model either by using the EM algorithm or by using differentiable components such as RBF functions or existing attention mechanisms.

\begin{figure}[t!]
	\centering
	\includegraphics[width=0.79\columnwidth]{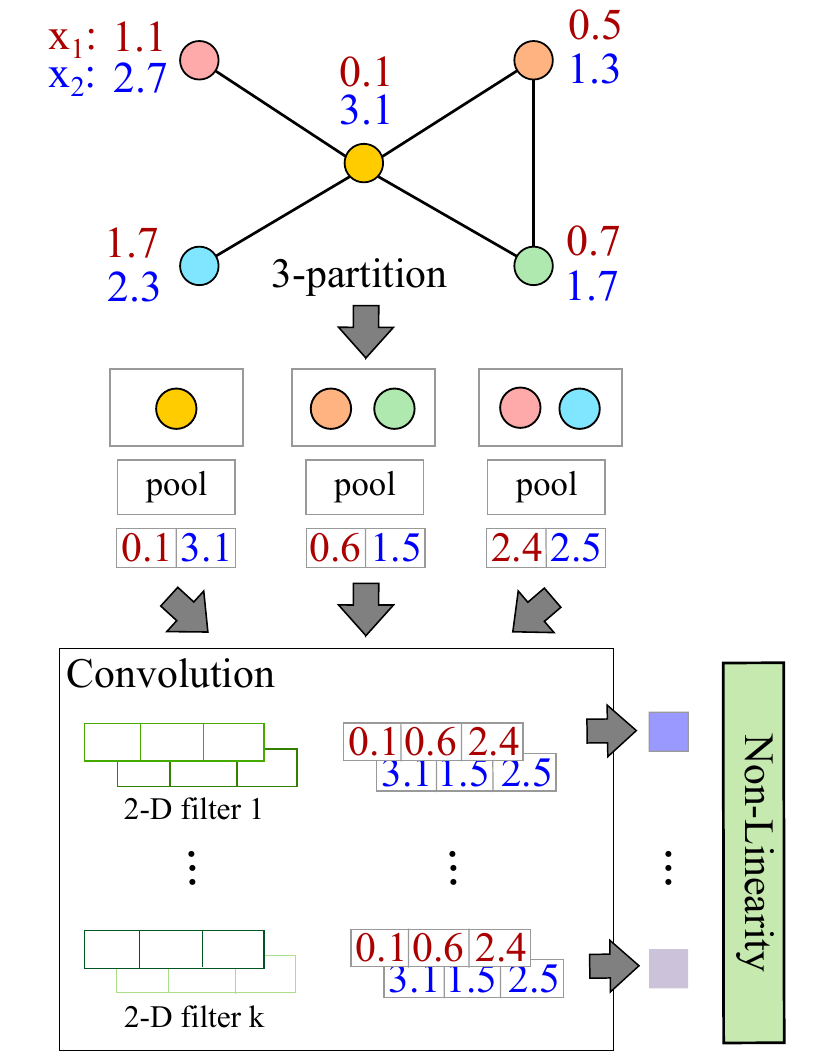}
	\caption{\label{fig-pool-partition} An illustration of a $3$-GCN where every 1-hop neighborhood is partitioned into exactly $3$ components. The aggregation operation is then performed  per component and attribute. The feature map of the following convolution can now operate on more than just one value per channel. }
\end{figure}

\subsection{Statistical Aggregation}

The aggregation component takes, for each node $v$, the partition $\{C_1, ..., C_c\}$ of $\{v\} \cup \mathbf{N}(v)$ computed in the previous step as input. It performs an aggregation operation for each of the node attributes and each component $C_i$, providing a statistical summary of the node attribute values. For instance, if there are two node attributes $x_1 \in \mathbb{R}$ and $x_2 \in \mathbb{R}$ and with average pooling, we compute, for each $j \in \{1, ..., c\}$,  $$x_1^{(j)} = \frac{1}{|C_j|}\sum_{v \in C_j} x_1(v) \mbox {\ \ and\  \ } x_2^{(j)} = \frac{1}{|C_j|}\sum_{v \in C_j} x_2(v),$$
where $x_i^{(j)}$ is the result of the pooling operation for partition component $j$ and attribute $i$, and $x_i(v)$ is the value of attribute $i$ at node $v$. One might want to chose a different aggregation operation for different attributes or use several aggregation functions and concatenate the results. The output of one aggregation operation is a tensor $\mathbf{B} \in \mathbb{R}^{|V| \times a \times c}$, where $a$ is the number of attributes and $c$ the number of components in the partitions.  

GCNs of Kipf \& Welling and MoNets always perform a (weighted) average aggregation which is determined by the global node degree. We believe that a combination of arbitrary pooling operations is useful in designing novel GCN variants. 

\subsection{Convolutions}

The input to the convolution is a tensor $\mathbf{B} \in \mathbb{R}^{|V| \times a \times c}$, where $a$ is the number of attributes. Hence, for every node in the graph, there is a matrix $\mathbf{B}(v) \in \mathbb{R}^{a \times c}$. The convolutional module maintains a set of $k$  2D filters of dimension $a \times c$ whose weights are learned during training. Each of these $k$ 2D  filters is applied to the matrix $\mathbf{B}(v)$ and results in a vector of dimension $k$. Hence, the output of the concolution is a feature map of dimension $|V| \times k$.
To this feature map one can apply a non-linear function element-wise such as the ReLU activation function as was done in the original GCN variant. 

As in previous work on GCNs, the number of layers, non-linearities, etc. can be chosen freely. In practice, a small number of layers has been shown to perform best.

\subsection{Formal Analysis}

As discussed before, we are motivated specifically by the lack of expressivity of the original GCN for semi-supervised learning on regular graphs. We can now show that the generalization is as expressive as a 2D convolutional net on image data \emph{without} an explicit encoding of spatial information. 
We can relate the $k$-GCN to CNNs for images in the following way. 
\begin{thmx}
Given an image in form of a grid graph.  A standard convolutional layer with receptive field size $(2m-1) \times (2m-1)$, no zero padding, and $k$ filters is identical (up to a fixed permutation) to the application of a $(2m-1)^2$-GCN with $k$ filters where the $(2m-1)^2$-partition is computed with a graph canonicalization algorithm.  
\end{thmx}

\begin{proof}
It is possible to show that if an input graph is a grid, then the partition of the neighbors resulting from a canonicalization algorithm such as \textsc{Nauty} is unique and identical for each of the vertices in the graph. 
\end{proof}

The computational overhead of $k$-GCNs is in the computation of the neighborhood clustering for each node. Depending on the labeling/canonicalization algorithm used, this can be more or less expensive. However, as has been shown before~\cite{niepert2016learning}, computing a canonicalization on smaller neighborhood graphs is highly efficient. Moreover, the partitions have to be computed only once as a preprocessing step. Once the partitions are computed, the computational complexity of a $k$-GCN is identical (up to the constant $k$) to that of a GCN.

With the ideas in this paper it is possible to define a spectrum of GCNs. On the one end, there is Kipf \& Welling GCNs where feature values of all 1-hop neighbors are average-pooled prior to the application of the convolutional filters. On one other end is the model where we chose $k$ to be the maximum node degree and we assign one node to each component of the clustering. Note that the proposed approach also generalizes recent methods for normalizing node neighborhoods~\cite{niepert2016learning}. We believe that while the best choice of clusters depends on the data, in numerous cases a more fine-grained partition can be beneficial. This is what we explore with ongoing experiments. 

\bibliographystyle{IEEEbib}
\bibliography{pgcn}

\end{document}